\documentclass{article}

\usepackage{microtype}
\usepackage{graphicx}
\usepackage{booktabs} 
\usepackage{hyperref}

\usepackage{amsmath}
\usepackage{float}
\usepackage{amssymb}
\usepackage{amsthm}
\usepackage{subcaption}
\usepackage{tabularx}
\newtheorem{theorem}{Theorem}[section]
\newtheorem{corollary}{Corollary}[theorem]
\newtheorem{lemma}[theorem]{Lemma}
\newtheorem{definition}{Definition}[section]
\DeclareMathOperator*{\argmin}{arg\,min}
\newcommand\T{\rule{0pt}{2.6ex}}       
\usepackage{verbatim}
\usepackage[section]{placeins}

\makeatletter
\AtBeginDocument{%
  \expandafter\renewcommand\expandafter\subsection\expandafter{%
    \expandafter\@fb@secFB\subsection
  }%
}
\makeatother

\usepackage[accepted]{icml2018}

\icmltitlerunning{FeTa: A DCA Pruning Algorithm with Generalization Error Guarantees}

\begin{document}

\twocolumn[
\icmltitle{FeTa: A DCA Pruning Algorithm with Generalization Error Guarantees}



\icmlsetsymbol{equal}{*}

\begin{icmlauthorlist}
\icmlauthor{Konstantinos Pitas}{epfl}
\icmlauthor{Mike Davies}{edinburgh}
\icmlauthor{Pierre Vandergheynst}{epfl}
\end{icmlauthorlist}

\icmlaffiliation{epfl}{LTS2 Laboratory, EPFL, Lausanne, Switzerland}
\icmlaffiliation{edinburgh}{IDCOM Laboratory, University of Edinburgh, Edinburgh, UK}

\icmlcorrespondingauthor{Pitas Konstantinos}{konstantinos.pitas@epfl.ch}

\icmlkeywords{Machine Learning, ICML}

\vskip 0.3in
]



\printAffiliationsAndNotice{\icmlEqualContribution} 

\begin{abstract}
Recent DNN pruning algorithms have succeeded in reducing the number of parameters in fully connected layers, often with little or no drop in classification accuracy. However, most of the existing pruning schemes either have to be applied during training or require a costly retraining procedure after pruning to regain classification accuracy. We start by proposing a cheap pruning algorithm for fully connected DNN layers based on difference of convex functions (DC) optimisation, that requires little or no retraining. We then provide a theoretical analysis for the growth in the Generalization Error (GE) of a DNN for the case of bounded perturbations to the hidden layers, of which weight pruning is a special case. Our pruning method is orders of magnitude faster than competing approaches, while our theoretical analysis sheds light to previously observed problems in DNN pruning. Experiments on commnon feedforward neural networks validate our results.  
\end{abstract}

\section{Introduction}
Recently, deep neural networks have achieved state-of-the art results in a number of machine learning tasks \citet{lecun2015deep}. Training such networks is computationally intensive and often requires dedicated and expensive hardware. Furthermore, the resulting networks often require a considerable amount of memory to be stored. Using a Pascal Titan X GPU the popular AlexNet and VGG-16 models require 13 hours and 7 days, respectively, to train, while requiring 200MB and 600MB, respectively, to store. The large memory requirements limit the use of DNNs in embedded systems and portable devices such as smartphones, which are now ubiquitous. 

A number of approaches have been proposed to reduce the DNN size during training time, often with little or no degradation to classification performance. Approaches include introducing bayesian, sparsity-inducing priors \citet{louizos2017bayesian} \citet{blundell2015weight} \citet{molchanov2017variational} and binarization \citet{hou2016loss} \citet{courbariaux2016binarized}.Other methods include the hashing trick used in \citet{chen2015compressing}, tensorisation \citet{novikov2015tensorizing} and efficient matrix factorisations \citet{yang2015deep}.

However, trained DNN models are used by researchers and developers that do not have dedicated hardware to train them, often as general feature extractors for transfer learning. In such settings it is important to introduce a \textit{cheap} compression method, i.e., one that can be implemented as a postprocessing step with little or no retraining. Some first work in this direction has been \citet{kim2015compression} \citet{han2015deep} \citet{han2015learning} although these still require a lengthy retraining procedure. Closer to our approach recently in \citet{aghasi2016net} the authors propose a convexified layerwise pruning algorithm termed Net-Trim. Building upon Net-Trim, the authors in \citet{dong2017learning} propose LOBS, an algorithm for layerwise pruning by loss function approximation. 

Pruning a neural network layer introduces a perturbation to the latent signal representations generated by that layer. As the pertubated signal passes through layers of non-linear projections, the perturbation could become arbitrarily large. DNN robustness to hidden layer perturbations has been investigated for random noise in \citet{raghu2016expressive}. For the case of pruning in \citet{aghasi2016net} and \citet{dong2017learning} the authors conduct a theoretical analysis using the Lipschitz properties of DNNs showing the stability of the latent representations, over the training set, after pruning. The methods employed have connections to recent work \citet{sokolic2017robust} \citet{bartlett2017spectrally} \citet{neyshabur2017pac} that have used the Lipschitz properties to analyze the Generalization Error (GE) of DNNs, a more useful performance measure.

\subsection{Contributions}
In this work we introduce a cheap pruning algorithm for dense layers of DNNs. We also conduct a theoretical analysis of how pruning affects the Generalization Error of the trained classifier. 
\begin{itemize}
\item We show that the sparsity-inducing objective proposed in \citet{aghasi2016net} can be cast as a difference of convex functions problem, that has an efficient solution. For a fully connected layer with input dimension $d_1$, output dimension $d_2$ and $N$ training samples, Net-Trim and LOBS scale like $\mathcal{O}(Nd_1^3)$ and $\mathcal{O}((N+d_2)d_1^2 )$, respectively. Our iterative algorithm scales like $\mathcal{O}(K(N+\frac{Nk}{N+\sqrt{k}}) \text{log}(\frac{1}{\epsilon}) d_1 d_2)$, where $\epsilon$ is the precision of the solution, $k$ is related to the Lipschitz and strong convexity constants, $d_2 \ll d_1$ and $K$ is the outer iteration number. Emprirically, our algorithm is orders of magnitude faster than competing approaches. We also extend our formulation to allow retraining a layer with any convex regulariser. 
\item We build upon the work of \citet{sokolic2017robust} to bound the GE of a DNN for the case of bounded perturbations to the hidden layer weights, of which pruning is a special case. Our theoretical analysis provides a principled way of pruning while managing the GE. In sharp contrast to the analysis of \citet{aghasi2016net} and \citet{dong2017learning} our analysis correctly predicts the previously observed phenomenon that accuracy degrades exponentially with the remaining depth of the pruned layer. 
\end{itemize}
Experiments on common feedforward architectures show that our method is orders of magnitude faster than competing pruning methods, while allowing for a controlled increase in GE. 

\subsection{Notation and Definitions}
We use the following notation in the sequel:matrices ,column vectors, scalars and sets are denoted by boldface upper-case letters ($\boldsymbol{X}$), boldface lower-case letters ($\boldsymbol{x}$), italic letters ($x$) and calligraphic upper-case letters ($\mathcal{X}$), respectively. The covering number of $\mathcal{X}$ with $d$-metric balls of radius $\rho$ is denoted by $\mathcal{N}(\mathcal{X};d,\rho)$. A $C_M$-regular $k$-dimensional manifold, where $C_M$ is a constant that captures "intrinsic" properties, is one that has a covering number $\mathcal{N}(\mathcal{X};d,\rho)=(\frac{C_M}{\rho})^k$.

\section{Our formulation}
\subsection{DC decomposition}
We consider a classification problem, where we observe a vector $\boldsymbol{x} \in \mathcal{X} \subseteq \mathbb{R}^N$ that has a corresponding class label $y \in \mathcal{Y}$. The set $\mathcal{X}$ is called the input space, $\mathcal{Y} = \{1,2,...,N_{\mathcal{Y}}\}$ is called the label space and $N_{\mathcal{Y}}$ denotes the number of classes. The samples space is denoted by $\mathcal{S}=\mathcal{X} \times \mathcal{Y}$ and an element of $\mathcal{S}$ is denoted by $s = (\boldsymbol{x},y)$. We assume that samples from $\mathcal{S}$ are drawn according to a probability distribution $P$ defined on $\mathcal{S}$. A training set of $m$ samples drawn from $P$ is denoted by $S_m = \{s_i\}^m_{i=1}=\{(\boldsymbol{x}_i,y_i)\}^m_{i=1}$.

We start from the Net-Trim formulation and show that it can be cast as a difference of convex functions problem. 
For each training signal $\boldsymbol{x} \in \mathbb{R}^{N}$ we assume also that we have access to the inputs $\boldsymbol{a} \in \mathbb{R}^{d_1} $ and the outputs $\boldsymbol{b} \in \mathbb{R}^{d_2} $ of the fully connected layer, with a rectifier non-linearity $\rho(x)=\textbf{\text{max}}(0,x)$. The optimisation problem that we want to solve is then

\begin{equation}
\min_{\boldsymbol{U}} \frac{1}{m} \sum_{s_j \in \mathcal{S}_m}||\rho(\boldsymbol{U}^{T}\boldsymbol{a}_j)-\boldsymbol{b}_j||^2_2+ \lambda \Omega (\boldsymbol{U}),
\end{equation} 

where $\lambda$ is the sparsity parameter. The term $||\rho(\boldsymbol{U}^{T}\boldsymbol{a}_j)-\boldsymbol{b}_j||^2_2$ ensures that the nonlinear projection remains the same for training signals. The term $ \lambda \Omega (\boldsymbol{U}) $ is the convex regulariser which imposes the desired structure on the weight matrix $\boldsymbol{U}$. 

The objective in Equation 1 is non-convex. We show that the optimisation of this objective can be cast as a difference of convex functions (DC) problem. We assume just one training sample $\boldsymbol{x} \in \mathbb{R}^{N}$, for simplicity, with latent representations $\boldsymbol{a} \in \mathbb{R}^{d} $ and $\boldsymbol{b} \in \mathbb{R}^{z} $

\begin{equation}
\begin{split}
& ||\rho(\boldsymbol{U}^{T}\boldsymbol{a})-\boldsymbol{b}||^2_2+ \lambda\Omega (\boldsymbol{U}) \\
& = \sum_{i}[\rho(\boldsymbol{u_i}^{T}\boldsymbol{a})-\boldsymbol{b_i} ]^2+\lambda \Omega (\boldsymbol{U}) \\
& = \sum_{i}[\rho^{2}(\boldsymbol{u_i}^{T}\boldsymbol{a})-2\rho(\boldsymbol{u_i}^{T}\boldsymbol{a})\boldsymbol{b_i}+\boldsymbol{b_i}^{2} ]+\lambda \Omega (\boldsymbol{U}) \\
& = \sum_{i}[ \rho^{2}(\boldsymbol{u_i}^{T}\boldsymbol{a})+\boldsymbol{b_i}^{2} ]+\lambda \Omega (\boldsymbol{U}) +\sum_{i}[-2\boldsymbol{b_i}\rho(\boldsymbol{u_i}^{T}\boldsymbol{a})] \\
& = \sum_{i}[ \rho^{2}(\boldsymbol{u_i}^{T}\boldsymbol{a})+\boldsymbol{b_i}^{2} ]+\lambda \Omega (\boldsymbol{U}) \\
& +\sum_{\substack {i \\ b_i<0} }[-2\boldsymbol{b_i}\rho(\boldsymbol{u_i}^{T}\boldsymbol{a})] + \sum_{\substack {i \\ b_i\geq0} }[-2\boldsymbol{b_i}\rho(\boldsymbol{u_i}^{T}\boldsymbol{a})].\\
\end{split}
\end{equation}

Notice that after the split the first term ($b_i < 0$) is convex while the second ($b_i \geq 0$) is concave. We note that $b_i \geq 0$ by definition of the ReLu and set

\begin{equation}
g(\boldsymbol{U};\boldsymbol{x}) = \sum_{i}[ \rho^{2}(\boldsymbol{u_i}^{T}\boldsymbol{a})+\boldsymbol{b_i}^{2} ],
\end{equation}

\begin{equation}
h(\boldsymbol{U};\boldsymbol{x}) = \sum_{\substack {i \\ b_i>0} }[2\boldsymbol{b_i}\rho(\boldsymbol{u_i}^{T}\boldsymbol{a})].
\end{equation}

Then by summing over all the samples we get

\begin{equation}
\begin{split}
f(\boldsymbol{U}) &= \sum_{j}g(\boldsymbol{U};\boldsymbol{x}_j)+\lambda \Omega (\boldsymbol{U}) - \sum_{j} h(\boldsymbol{U};\boldsymbol{x}_j) \\
				  &= g(\boldsymbol{U})+\lambda \Omega (\boldsymbol{U}) - h(\boldsymbol{U}), \\
\end{split}
\end{equation}

which is difference of convex functions. The rectifier nonlinearity is non-smooth, but we can alleviate that by assuming a smooth approximation. A common choice for this task is $\rho(x) = \frac{1}{\beta}\text{log}(1+\text{exp}(\beta x))$, with $\beta$ a positive constant.

\subsection{Optimisation}
It is well known that DC programs have efficient optimisation algorithms. We propose to use the DCA algorithm \citet{tao1997convex}. DCA is an iterative algorithm that consists in solving, at each iteration, the convex optimisation problem obtained by linearizing $h(\cdot)$ (the non-convex part of $f = g - h$) around the current solution. Although DCA is only guaranteed to reach local minima the authors of \citet{tao1997convex} state that DCA often converges to the global minimum, and has been used succefully to optimise a fully connected DNN layer \citet{fawzi2015dictionary}. At iteration $k$ of DCA, the linearized optimisation problem is given by

\begin{equation}
\argmin_{\boldsymbol{U}}\{g(\boldsymbol{U})+\lambda \Omega (\boldsymbol{U})-Tr(\boldsymbol{U}^{T}\nabla h(\boldsymbol{U}^k))\},
\end{equation}

where $\boldsymbol{U}^{k}$ is the solution estimate at iteration $k$. The detailed procedure is then given in algorithms 1 and 2. We assume that the regulariser is convex but possibly non-smooth in which case the optimisation can be performed using proximal methods.

\begin{algorithm}[h!] 
\caption{FeTa (Fast and Efficient Trimming Algorithm)}
\label{alg:algorithm1}
\begin{algorithmic}[1]
\STATE Choose initial point: $\boldsymbol{U}^0$
\FOR {k = 1,...,K}
	\STATE Compute $C \gets \nabla h(\boldsymbol{U}^k)$.
	\STATE Solve with Algorithm 2 the convex optimisation problem:
	\begin{equation}
	\boldsymbol{U}^{k+1} \gets \argmin_{\boldsymbol{U}}\{g(\boldsymbol{U})+\lambda \Omega (\boldsymbol{U})-Tr(\boldsymbol{U}^{T}C)\}
	\end{equation}
\ENDFOR
\STATE If $\boldsymbol{U}^{k+1} \approx \boldsymbol{U}^{k}$ return $\boldsymbol{U}^{k+1}$.
\end{algorithmic}
\end{algorithm}

\begin{algorithm}[h!] 
\caption{Acc-Prox-SVRG}
\label{alg:algorithm3}
\begin{algorithmic}[1]
\STATE \textbf{Initialization}: $\tilde{\boldsymbol{x} }_0 \gets \boldsymbol{U}^k , \beta , \eta $
\FOR {s = 1,...,S}
	\STATE $\tilde{\boldsymbol{u} } = \nabla g(\tilde{\boldsymbol{x} }_s)$
	\STATE $\boldsymbol{x}_1 = \boldsymbol{y}_1 = \tilde{\boldsymbol{x} }_s$
	\FOR {t = 1,2,...,T}
		\STATE Choose $(\boldsymbol{A},\boldsymbol{B})$ randomly chosen minibatch.
		\STATE $\boldsymbol{u}_t = \nabla g_{\boldsymbol{A},\boldsymbol{B}}(\boldsymbol{y}_t) - \nabla g_{\boldsymbol{A},\boldsymbol{B}}(\tilde{\boldsymbol{x} }_s)+\tilde{\boldsymbol{u} }$
		\STATE $\boldsymbol{x}_{t+1} = \text{prox}_{\eta h}(\boldsymbol{y}_t - \eta \boldsymbol{u}_t)$
		\STATE $\boldsymbol{y}_{t+1} = \boldsymbol{x}_{t+1} + \beta(\boldsymbol{x}_{t+1}-\boldsymbol{x}_t)$
	\ENDFOR
	\STATE $\tilde{\boldsymbol{x} }_{s+1} = \boldsymbol{x}_{T+1}$
\ENDFOR
\STATE Return $\boldsymbol{U}^{k+1} \gets \tilde{\boldsymbol{x} }_{S+1}$
\end{algorithmic}
\end{algorithm}


In order to solve the linearized problem we propose to use Accelerated Proximal SVRG (Acc-Prox-SVRG), which was presented in \citet{nitanda2014stochastic}. We detail this method in Algorithm 2b. At each iteration a minibatch $\boldsymbol{A}$ and $\boldsymbol{B}$ is drawn. The gradient for the smooth part is calculated and the algorithm takes a step in that direction with step size $\eta$. Then the proximal operator for the non-smooth regulariser $\lambda \Omega(\cdot)$ is applied to the result. The hyperparameters for Acc-Prox-SVRG are the acceleration parameter $\beta$ and the gradient step $\eta$. We have found that in our experiments, using $\beta = 0.95$ and $\eta \in \{0.001 , 0.0001 \}$ gives the best results. 

We name our algorithm FeTa, Fast and Efficient Trimming Algorithm.

\section{Generalization Error}

\subsection{Generalization Error of Pruned Layer}
Having optimized our pruned layer for the training set we want to see if it is stable for the test set. We denote $f^1(\cdot,\boldsymbol{W}^1)$ the original representation and $f^2(\cdot,\boldsymbol{W}^2)$ the pruned representation.  We assume that after training $\forall s_i \in \mathcal{S}_m \: ||f^1(\boldsymbol{a_i},\boldsymbol{W}^1)-f^2(\boldsymbol{a_i},\boldsymbol{W}^2)||_2^2 \leq C_1$. Second, we assume that $\forall s \in \mathcal{S} \; \exists s_i \in \mathcal{S}_m \Rightarrow ||a-a_i||^2_2 \leq \epsilon $. Third, the linear operators in $\boldsymbol{W}^1$ , $\boldsymbol{W}^2$ are frames with upper frame bounds $B_1$ , $B_2$ respectively. 

\begin{theorem}
For any testing point $s \in \mathcal{S}$, the distance between the original representation $f^1(\boldsymbol{a},\boldsymbol{W}^1)$ and the pruned representation $f^2(\boldsymbol{a},\boldsymbol{W}^2)$ is bounded by $||f^1(\boldsymbol{a},\boldsymbol{W}^1)-f^2(\boldsymbol{a},\boldsymbol{W}^2)||^2_2 \leq C_2$ where $C_2 = C_1 + (B_1+B_2)\epsilon$.
\end{theorem}

the detailed proof can be found in Appendix A. 

\subsection{Generalization Error of Classifier}
In this section we use tools from the robustness framework \citet{xu2012robustness} to bound the generalization error of the new architecture induced by our pruning. We consider DNN classifiers defined as

\begin{equation}
g(\boldsymbol{x}) = \max_{i \in [N_y] } (f(\boldsymbol{x}))_i ,
\end{equation}

where $(f(\boldsymbol{x}))_i$ is the $i-$th element of $N_{y}$ dimensional output of a DNN $f:\mathbb{R}^N \rightarrow \mathbb{R}^{N_y}$. We assume that $f(\boldsymbol{x})$ is composed of $L$ layers

\begin{equation}
f(\boldsymbol{x})=f_L(f_{L-1}(...f_1(\boldsymbol{x},\boldsymbol{W}_1),...\boldsymbol{W}_{L-1}),\boldsymbol{W}_L) ,
\end{equation}

where $f_l(\cdot,\boldsymbol{W}_l)$ represents the $l-$th layer with parameters $\boldsymbol{W}_l$, $l = 1,...,L$. The output of the $l-$th layer is denoted $\boldsymbol{z}^l$, i.e. $\boldsymbol{z}^l=f_l(\boldsymbol{z}^{l-1},\boldsymbol{W}_l)$. The input layer corresponds to $\boldsymbol{z}^{0} = \boldsymbol{x}$ and the output of the last layer is denoted by $\boldsymbol{z} = f(\boldsymbol{x})$. We then need the following two definitions of the classification margin and the score that we take from \citet{sokolic2017robust}. These will be useful later for measuring the generalization error.

\begin{definition}
(\normalfont{Score}). For a classifier $g(\boldsymbol{x})$ a training sample $s_i = (\boldsymbol{x}_i,y_i)$ has a score

\begin{equation}
o(s_i)=o(\boldsymbol{x}_i,g(\boldsymbol{x}_i))=\min_{j \neq g(\boldsymbol{x}_i)}\sqrt{2}(\delta_{g(\boldsymbol{x}_i)}-\delta_{j})^{T}f(\boldsymbol{x}_i),
\end{equation}

where $\delta_i \in \mathcal{R}^{N_y}$ is the Kronecker delta vector with $(\delta_i)_i=1$, and $g(\boldsymbol{x}_i)$ is the output class for $s_i$ from classifier $g(\boldsymbol{x})$ which can also be $g(\boldsymbol{x}_i) \neq y_i$.
\end{definition}

\begin{definition}
(\normalfont{Training Sample Margin}). For a classifier $g(\boldsymbol{x})$ a training sample $s_i = (\boldsymbol{x}_i,y_i)$ has a classification margin $\gamma(s_i)$ measured by the $l_2$ norm if
\begin{equation}
g(\boldsymbol{x})=g(\boldsymbol{x}_i); \;\;\; \forall \boldsymbol{x} : ||\boldsymbol{x}-\boldsymbol{x}_i||_2< \gamma(s_i).
\end{equation}
\end{definition}

The classification margin of a training sample $s_i$ is the radius of the largest metric ball (induced by the $l_2$ norm) in $\mathcal{X}$ centered at $\boldsymbol{x}_i$ that is contained in the decision region associated with the classification label $g(\boldsymbol{x}_i)$. Note that it is possible for a classifier to misclassify a training point $g(\boldsymbol{x}_i) \neq y_i$. We then restate a useful result from \citet{sokolic2017robust}. 

\begin{corollary}
Assume that $\mathcal{X}$ is a (subset of) $C_M$-regular k-dimensional manifold, where $\mathcal{N}(\mathcal{X};d;\rho) \leq (\frac{C_M}{\rho})^k$. Assume also that the DNN classifier $g(\boldsymbol{x})$ achieves a lower bound to the classification score $o(\tilde{s}) < o(s_i), \; \forall s_i \in S_m$ and take $l(g(\boldsymbol{x}_i),y_i)$ to be the $0-1$ loss. Then for any $\delta > 0$, with probability at least $1-\delta$,

\begin{equation}
\text{GE}(g) \leq A \cdot (\gamma)^{-\frac{k}{2}}+B,
\end{equation}
where $A = \sqrt{ \frac{\log{(2)} \cdot N_y \cdot 2^{k+1} \cdot (C_M)^k}{ m } }$ and $B = \sqrt {\frac{2\log{1/\delta}}{m}}$ can be considered constants related to the data manifold and the training sample size, and $\gamma = \frac{o(\tilde{s})}{\prod_i ||\boldsymbol{W}_i||_2 }$.
\end{corollary}

We are now ready to state our main result.

\begin{theorem}
Assume that $\mathcal{X}$ is a (subset of) $C_M$-regular k-dimensional manifold, where $\mathcal{N}(\mathcal{X};d;\rho) \leq (\frac{C_M}{\rho})^k$. Assume also that the DNN classifier $g_1(\boldsymbol{x})$ achieves a lower bound to the classification score $o(\tilde{s}) < o(s_i), \; \forall s_i \in S_m$ and take $l(g(\boldsymbol{x}_i),y_i)$ to be the $0-1$ loss. Furthermore assume that we prune classifier $g_1(\boldsymbol{x})$ on layer $i_{\star}$ using Algorithm 1, to obtain a new classifier $g_2(\boldsymbol{x})$. Then for any $\delta > 0$, with probability at least $1-\delta$, when $(\gamma-\sqrt{C_2} \cdot \frac{ \prod_{i > i_{\star}}||\boldsymbol{W}_i||_2}{ \prod_i||\boldsymbol{W}_i||_2}) > 0$,

\begin{equation}
\text{GE}(g_2) \leq A \cdot (\gamma-\sqrt{C_2} \cdot \frac{ \prod_{i > i_{\star}}||\boldsymbol{W}_i||_2}{ \prod_i||\boldsymbol{W}_i||_2})^{-\frac{k}{2}}+B,
\end{equation}
where $A = \sqrt{ \frac{\log{(2)} \cdot N_y \cdot 2^{k+1} \cdot (C_M)^k}{ m } }$ and $B = \sqrt {\frac{2\log{1/\delta}}{m}}$ can be considered constants related to the data manifold and the training sample size, and $\gamma = \frac{o(\tilde{s})}{\prod_i ||\boldsymbol{W}_i||_2 }$.
\end{theorem}

The detailed proof can be found in Appendix B. The bound depends on two constants related to intrinsic properties of the data manifold, the regularity constant $C_M$ and the intrinsic data dimensionality $k$. In particular the bound depends exponentially on the intrinsic data dimensionality $k$. Thus more complex datasets are expected to lead to less robust DNNs. This has been recently observed empirically in \citet{bartlett2017spectrally}. The bound also depends on the spectral norm of the hidden layers $||\boldsymbol{W}_i ||_2$. Small spectral norms lead to a larger base in $(\cdot)^{-\frac{k}{2}} $ and thus to tigher bounds. 

With respect to pruning our result is quite pessimistic as the pruning error $\sqrt{C_2}$ is multiplied by the factor $\prod_{i > i_{\star}}||\boldsymbol{W}_i||_2$. Thus in our analysis the GE grows exponentially with respect to the remaining layer depth of the pertubated layer. This is in line with previous work \citet{raghu2016expressive} \citet{han2015learning} that demonstrates that layers closer to the input are much less robust compared to layers close to the output. Our algorithm is applied to the fully connected layers of a DNN, which are much closer to the output compared to convolutional layers. 

We can extend the above bound to include pruning of multiple layers.

\begin{theorem}
Assume that $\mathcal{X}$ is a (subset of) $C_M$-regular k-dimensional manifold, where $\mathcal{N}(\mathcal{X};d;\rho) \leq (\frac{C_M}{\rho})^k$. Assume also that the DNN classifier $g_1(\boldsymbol{x})$ achieves a lower bound to the classification score $o(\tilde{s}) < o(s_i), \; \forall s_i \in S_m$ and take $l(g(\boldsymbol{x}_i),y_i)$ to be the $0-1$ loss. Furthermore assume that we prune classifier $g_1(\boldsymbol{x})$ on all layers using Algorithm 1, to obtain a new classifier $g_2(\boldsymbol{x})$. Then for any $\delta > 0$, with probability at least $1-\delta$, when $(\gamma- \frac{ \sum_{i=0}^L \sqrt{C_{i2} } \prod_{j=i+1}^L||\boldsymbol{W}_j||_2}{ \prod_i||\boldsymbol{W}_i||_2}) > 0$,

\begin{equation}
\text{GE}(g_2) \leq A \cdot (\gamma-\frac{ \sum_{i=0}^L \sqrt{C_{i2} } \prod_{j=i+1}^L||\boldsymbol{W}_j||_2}{ \prod_i||\boldsymbol{W}_i||_2})^{-\frac{k}{2}}+B,
\end{equation}
where $A = \sqrt{ \frac{\log{(2)} \cdot N_y \cdot 2^{k+1} \cdot (C_M)^k}{ m } }$ and $B = \sqrt {\frac{2\log{1/\delta}}{m}}$ can be considered constants related to the data manifold and the training sample size, and $\gamma = \frac{o(\tilde{s})}{\prod_i ||\boldsymbol{W}_i||_2 }$.
\end{theorem}

The detailed proof can be found in Appendix C. The bound predicts that when pruning multiple layers the GE will be much greater than the sum of the GEs for each individual pruning. We note also the generality of our result; even though we have assumed a specific form of pruning, the GE bound holds for any type of bounded perturbation to a hidden layer.

\section{Experiments}

We make a number of experiments to compare FeTa with LOBS and NetTrim-ADMM. All experiments were run on a MacBook Pro with CPU 2.8GHz Intel Core i7 and RAM 16GB 1600 MHz DDR3.
\subsection{Time Complexity}
First we compare the execution time of FeTa with that of LOBS and NetTrim-ADMM. We set $\Omega (\boldsymbol{U}) = ||\boldsymbol{U}||_1$ and aim for $95\%$ sparsity. We set $d_1$ to be the input dimensions, $d_2$ to be the output dimensions and $N$ to be the number of training samples. Assuming that each $g(\boldsymbol{U};\boldsymbol{x}_j)$ is $L$-Lipschitz smooth and $g(\boldsymbol{U})$ is $\mu$-strongly convex, if we optimise for an $\epsilon$ optimal solution and set $k = L/\mu$, $\text{FeTa}$ scales like $\mathcal{O}(K(N+\frac{Nk}{N+\sqrt{k}}) \text{log}(\frac{1}{\epsilon}) d_1 d_2)$. We obtain this by multiplying the number of outer iterations $K$ with the number of gradient evaluations required to reach an $\epsilon$ good solution in inner Algorithm 2, and finally multiplying with the gradient evaluation cost. Conversely LOBS scales like $\mathcal{O}((N+d_2)d_1^2)$ while NetTrim-ADMM scales like $\mathcal{O}(N d_1^3)$ due to the required Cholesky factorisation. This gives a computational advantage to our algorithm in settings where the input dimension is large. We validate this by constructing a toy dataset with $d_2 =10$ , $d_1 =\{2000:100:3000\}$ and $N =1000$. The samples $\boldsymbol{a} \in \mathbb{R}^{d1}$ and $\boldsymbol{b} \in \mathbb{R}^{d2}$ are generated with $i.i.d$ Gaussian entries. We plot in Figure 1 the results, which are in line with the theoretical predictions.


\subsection{Classification Accuracy}

\begin{figure}[t!]
\includegraphics[scale = 0.55]{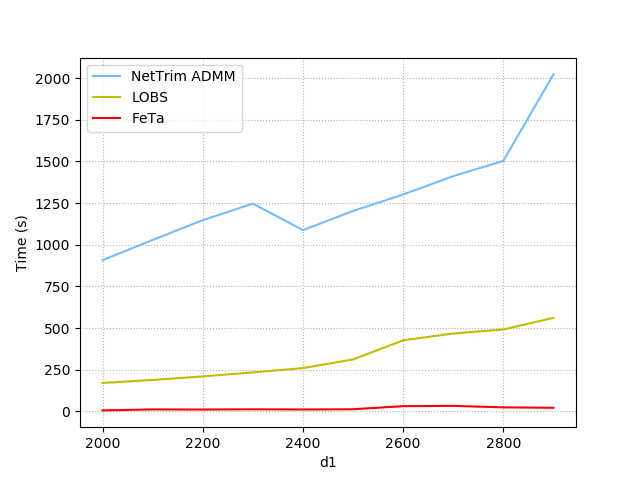}
\centering
\caption{\textbf{Time Complexity}: We plot the calculation time for FeTa, NetTrim and LOBS for the toy dataset. We see that the computation time is in line with theoretical predictions. FeTa scales roughly as $\mathcal{O}(Nd_2d_1)$ while NetTrim and LOBS scale like $\mathcal{O}(Nd_2^3)$ and $\mathcal{O}(Nd_2^2)$. As the size of the input dimensions increases FeTa becomes orders of magnitude faster than the competing approaches.}
\end{figure}

\subsubsection{Sparse Regularisation}
In this section we perform experiments on the proposed compression scheme with feedforward neural networks. We compare the original full-precision network (without compression) with the following compressed networks: (i) $\text{FeTa}$ with $\Omega (\boldsymbol{U}) = ||\boldsymbol{U}||_1$ (ii) Net-Trim (iii) LOBS (iv) Hard Thresholding. We refer to the respective papers for Net-Trim and LOBS. Hard Thresholding is defined as $F(\boldsymbol{x})=\boldsymbol{x} \odot I(|\boldsymbol{x}|>t)$, where $I$ is the elementwise indicator function, $\odot$ is the Hadamard product and $t$ is a positive constant. 

Experiments were performed on two commonly used datasets:
\begin{enumerate}
	\item \textit{MNIST}: This contains $28 \times 28$ gray images from ten digit classes. We use 55000 images for training, another 5000 for validation, and the remaining 10000 for testing. We use the LeNet-5 model:
	
	\begin{equation}
	\begin{split}
		&\text{Input} \rightarrow (1 \times 6C5)  \rightarrow MP2 \rightarrow (6 \times 16C5)  \\ 
		& \rightarrow MP2 \rightarrow 120FC \rightarrow 84FC \rightarrow 10SM \rightarrow \text{Output},	
	\end{split}
	\end{equation}
	where $C5$ is a $5 \times 5$ ReLU convolution layer, $MP2$ is a $2 \times 2$ max-pooling layer, $FC$ is a fully connected layer and $SM$ is a linear softmax layer. 

	\item \textit{CIFAR-10}:This contains 60000 $32 \times 32$ color images for ten object classes. We use 50000 images for training and the remaining 10000 for testing. The training data is augmented by random cropping to $24 \times 24$ pixels, random flips from left to right, contrast and brightness distortions to 200000 images. We use a smaller variant of the AlexNet model:
	
	\begin{equation}
	\begin{split}
		&\text{Input} \rightarrow (3 \times 64C5) \rightarrow MP2 \rightarrow (64 \times 64C5) \\
		& \rightarrow MP2 \rightarrow 384FC \rightarrow 192FC \rightarrow 10SM \rightarrow \text{Output}.
	\end{split}	
	\end{equation}

	
\end{enumerate}

We first prune \textbf{only the first} fully connected layer (the one furthest from the output) for clarity. Figure 2 shows the classification accuracy vs compression ratio for $\text{FeTa}$, $\text{NetTrim}$, LOBS and Hard Thresholding. We see that Hard Thresholding works adequately up to $85\%$ sparsity. From this level of sparsity and above the performance of Hard Thresholding degrades rapidly, FeTa has $\boldsymbol{10\%}$ higher accuracy on average while being the same or marginally worse than LOBS and NetTrim.



\begin{figure*}[t!]
\centering
\begin{subfigure}{.5\textwidth}
  \centering
  \includegraphics[scale=0.55]{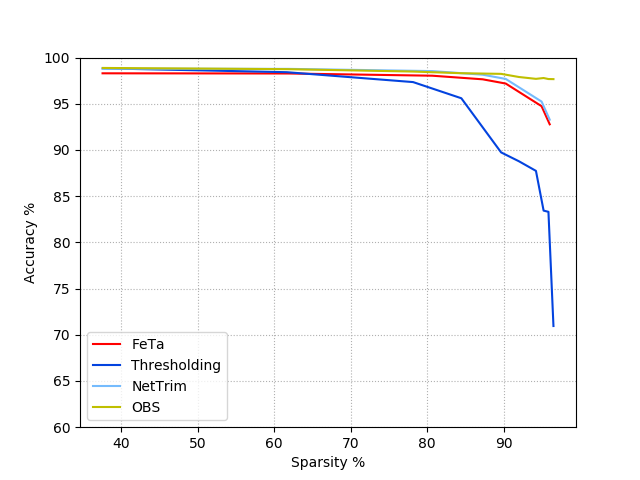}
  \caption{}
\end{subfigure}%
\begin{subfigure}{.5\textwidth}
  \centering
  \includegraphics[scale=0.55]{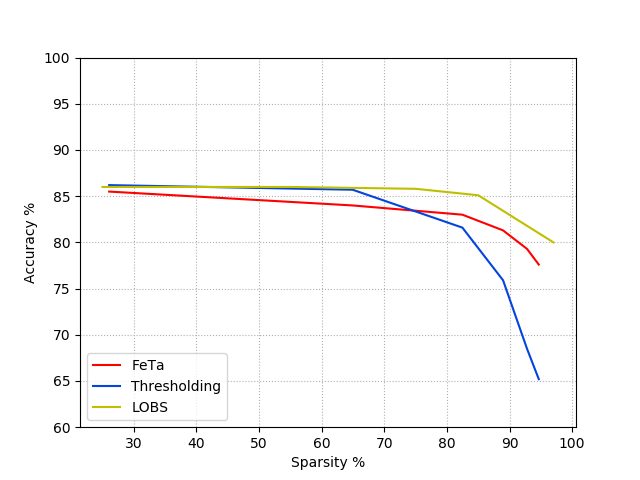}
  \caption{}
\end{subfigure}
\caption{\textbf{Accuracy vs Sparsity}: (a)We plot the classification accuracy of the pruned LeNet-5 architecture for different sparsity levels. Until the 80\% sparsity level roughly all methods are equal. For sparsity levels greater than 80\% FeTa clearly outperforms Hard Thresholding while remaining competitive with LOBS. (b)We plot the classification accuracy of the pruned CifarNet architecture for different sparsity levels. The results are consistent with the LeNet-5 experiment. }
\end{figure*}

For the task of pruning the first fully connected layer we also show detailed comparison results for all methods in Table 1. For the LeNet-5 model, FeTa achieves the same accuracy as Net-Trim while being $25 \times$ faster. This is expected as the two algorithms optimise a similar objective, while FeTa exploits the structure of the objective to achieve lower complexity in optimisation. Furthermore FeTa achieves marginally lower classification accuracy compared to LOBS while being $5 \times$ faster, and is significantly better than Thresholding. 

\begin{table}[h!]
\caption{Test accuracy rates (\%) prune only first fully connected layer.} \label{tab:title2} 
\label{sample-table}
\vskip 0.15in
\begin{center}
\begin{small}
\begin{sc}
\begin{tabular}{ lcccc  }
  \toprule
  LeNet-5 & Original & CR & Pruned & Time \\ 
  \midrule
  Net-Trim & 99.2\% & 95\% & 95\% & 455s \T\\		
  LOBS  & 99.2\% & 95\% & 97\% & 90s \\
  Threshold  & 99.2\% & 95\% & 83\% & - \\
  $\textbf{FeTa}$  & 99.2\% & 95\% & $\boldsymbol{95\%}$ & $\boldsymbol{18}$\textbf{s} \\
  \midrule
  CifarNet & Original & CR & Pruned & Time \\ 
  \midrule
  Net-Trim  & 86\% & - & - & -  \T\\
  LOBS  & 86\% & 90\% & 83.4\% & 3h 15min \\
  Threshold  & 86\% & 90\% & 73\% & - \\
  $\textbf{FeTa}$  & 86\% & 90\% & $\boldsymbol{80\%}$ & $\boldsymbol{20}$\textbf{min} \\
  \bottomrule
\end{tabular}
\end{sc}
\end{small}
\end{center}
\vskip -0.1in
\end{table}

\begin{table}[h!]
\captionof{table}{Test accuracy rates (\%) prune all fully connected layers.} \label{tab:title2} 
\label{sample-table}
\vskip 0.15in
\begin{center}
\begin{small}
\begin{sc}
\begin{tabular}{ lccccc  }
  \toprule
  LeNet-5 &  Original & CR & Pruned &  Time  \\ 
  \midrule
  Net-Trim & 99.2\% & 90\% & 95\% & 500s \\		
  LOBS  & 99.2\% & 90\% & 97\% & 97s \\
  Threshold  & 99.2\% & 90\% & 64\% & - \\
  $\textbf{FeTa}$ & 99.2\% & 90\% & $\boldsymbol{95\%}$ & $\boldsymbol{38}$\textbf{s} \\
  \midrule
  CifarNet & Original & CR & Pruned & Time \\ 
  \midrule
  Net-Trim & 86\% & - & - & -  \\
  LOBS  & 86\% & 90\% & 83.4\% & 3h 15min \\
  Threshold & 86\% & 90\% & 64\% & - \\
  $\textbf{FeTa}$ & 86\% & 90\% & $\boldsymbol{71\%}$ & $\boldsymbol{25}$\textbf{min}  \\
  \bottomrule
\end{tabular}
\end{sc}
\end{small}
\end{center}
\vskip -0.1in
\end{table}


For the CifarNet model we see in Table 1 that Net-Trim is not feasible on the machine used for the experiments as it requires over 16GB of RAM. Compared to LOBS FeTa again achieves marginally lower accuracy but is $8\times$ faster.

Next we prune both the fully connected layers in the two architectures to the same sparsity level and show the results in Table 2. We lower the achieved sparsity for all methods to $90\%$. For MNIST The accuracy results are the same as pruning a single layer, with FeTa achieving the same or marginally worse results while being $13\times$ faster than Net-Trim and $2.5\times$ faster than LOBS. For the Cifar experiment FeTa shows a bigger degradation in performance compared to LOBS while remaining $8\times$ faster. Thresholding achieves a notably bad result of $\boldsymbol{64\%}$ accuracy, which makes the method essentially inapplicable for multilayer pruning.

We note here that the degraded performance of FeTa for two layer pruning in Cifar is due to a poor solution for the second dense layer. By combining FeTa for the first dense layer and Thresholding for the second dense layer one can achieve $\boldsymbol{77\%}$ accuracy for the same computational cost.  Furthermore as mentioned in \citet{dong2017learning} and \citet{wolfe2017incredible} retraining can recover classification accuracy that was lost during pruning. Starting from a good pruning which doesn't allow for much degradation significantly reduces retraining time.

\subsubsection{Low Rank Regularisation}
As a proof of concept for the generality of our approach we apply our method while imposing low-rank regularisation on the learned matrix $\boldsymbol{U}$. For low rank $k$ we compare two methods (i) $\text{FeTa}$ with $\Omega (\boldsymbol{U}) = ||\boldsymbol{U}||_{\star}$ and optimised with Acc-Prox-SVRG and (ii) Hard Thresholding of singular values using the truncated SVD defined as $\boldsymbol{U} = \boldsymbol{N \Sigma} \boldsymbol{V}^{\star}, \; \boldsymbol{\Sigma} = \text{diag}(\{\sigma_i\}_{1 \leq i \leq k})$. We plot the results in Figure 3. 

\begin{figure}[h!]
\centering
\begin{subfigure}{.24\textwidth}
  \centering
  \includegraphics[scale=0.28]{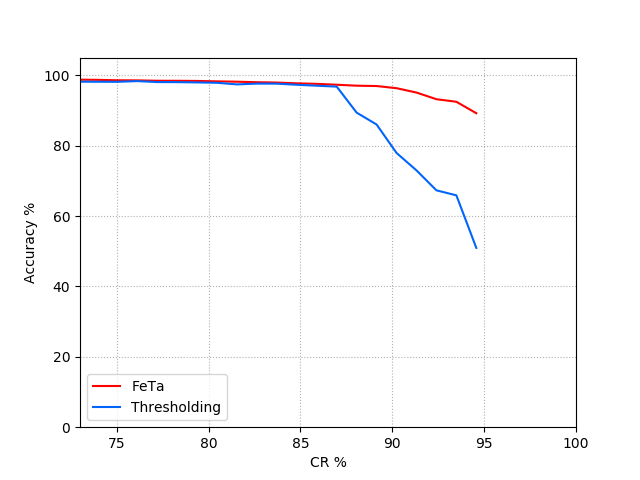}
  \caption{LeNet-5}
\end{subfigure}%
\begin{subfigure}{.24\textwidth}
  \centering
  \includegraphics[scale=0.28]{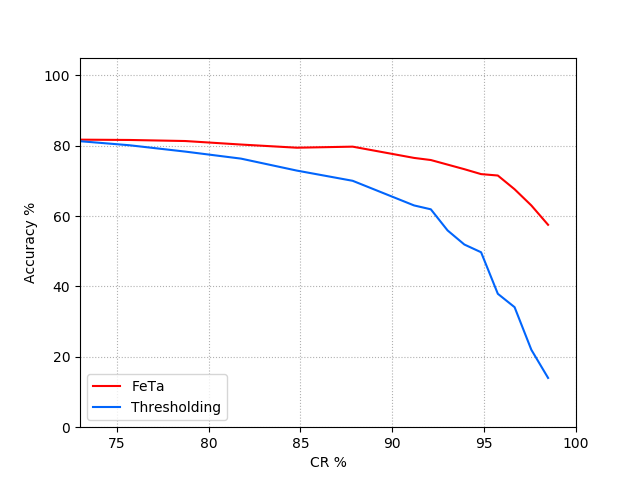}
  \caption{CifarNet}
\end{subfigure}
\caption{\textbf{Accuracy vs CR}: (a)We plot the classification accuracy of the low-rank compressed LeNet-5 architecture for different CR levels. Until the 85\% CR level roughly all methods are equal. For CR levels greater than 85\% FeTa clearly outperforms Hard Thresholding. (b)We plot the classification accuracy of the pruned CifarNet architecture for different CR levels. The results are consistent with the LeNet-5 experiment. } 
\end{figure}

In the above given $\boldsymbol{U} \in \mathbb{R}^{d_1 \times d_2}$ the Commpression Ratio (CR) is defined as $\text{CR} = (k*d_1 + k + k*d_2)/(d_1 * d_2)$. The results are in line with the $l_1$ regularisation, with significant degredation in classification accuracy for Hard Thresholding above $85\%$ CR.

\subsection{Generalization Error}

\begin{figure*}[t!]
\centering
\begin{subfigure}{.5\textwidth}
  \includegraphics[scale = 0.55]{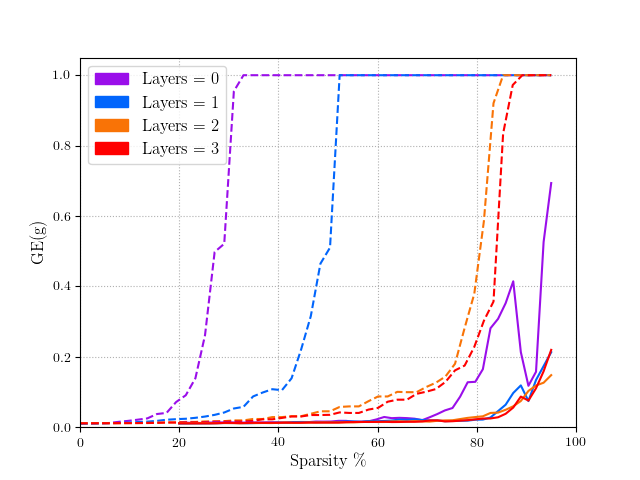}
  \centering
  \caption{Single Layer}
\end{subfigure}%
\begin{subfigure}{.5\textwidth}
  \includegraphics[scale = 0.55]{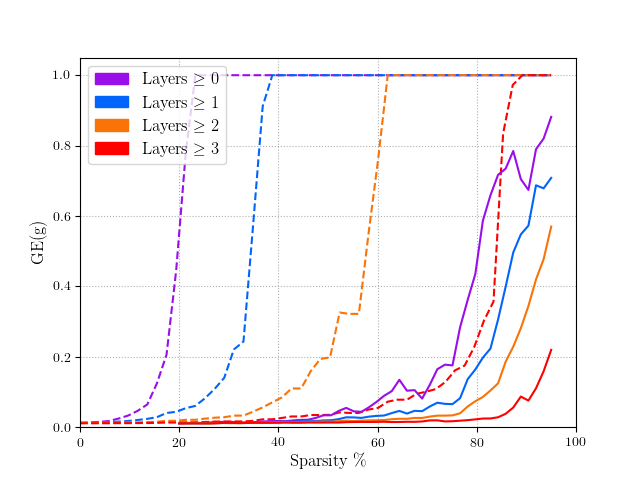}
  \centering
  \caption{Multiple Layers}
\end{subfigure}
\caption{\textbf{Layer Robustness}: We plot the theoretical prediction for the GE (dashed lines) and the empirical value of the GE (solid lines) for single layer pruning (a) and multilayer (b) pruning. Our theoretical predictions are tight for layers with small remaining depth but are loose for layers with big remaining depth. We first focus on pruning for $80\%$ sparsity. Layer $i=0$ is as predicted exponentially less robust compared to layers $i=\{1,2,3\}$. We then focus on pruning layer $i=0$ and layers $i\geq0$ for 80\% sparsity. We see that even though the GE errors for $i>0$ are negligible the GE error for $i\geq0$ is exponentially greater than the sum of the GEs when pruning $i=0$ and $i>0$. Interestingly in the empirical GE estimate there exists an artifact around 90\% sparsity which is partially captured by our prediction.   }
\end{figure*}

According to our theoretical analysis the GE drops exponentially with remaining layer depth. To corroborate this we train a LeNet-5 to high accuracy, then we pick a single layer and gradually increase its sparsity using Hard Thresholding. We find that the layers closer to the input are exponentially less robust to pruning, in line with our theoretical analysis. We plot the results in Figure 4.a. For some layers there is a sudden increase in accuracy around $90\%$ sparsity which could be due to the small size of the DNN. We point out that in empirical results \citet{raghu2016expressive} \citet{han2015learning} for much larger networks the degradation is entirely smooth.

Next we test our multilayer pruning bound. We prune to the same sparsity levels all layers in the sets $i \geq 0$ , $i \geq 1$ , $i \geq 2$ , $i \geq 3$. We plot the results in Figure 4.b. It is evident that the accuracy loss for layer groups is not simply the addition of the accuracy losses of the individual layers, but shows an exponential drop in accordance with our theoretical result.

We now aim to see how well our bound captures this exponential behaviour. We take two networks $g_a$ pruned at layer 3 and an unpruned network $g_b$ and make a number of simplifying assumptions. First we assume that in Theorem 3.3 $B=0$ such that $\text{GE}(g_{\star}) \leq A \cdot (\gamma-\frac{ \sum_{i=0}^L \sqrt{C_{i2} } \prod_{j=i+1}^L||\boldsymbol{W}_j||_2}{ \prod_i||\boldsymbol{W}_i||_2})^{-\frac{k}{2}}$. This is logical as $B$ includes only log terms. Assuming that the bounds are tight we now aim to calculate

\begin{equation}
\begin{split}
& \frac{\text{GE}(g_{a})}{\text{GE}(g_{b})} = \left( \frac{\gamma - \sum_{i=0}^L (\sqrt{C_{i2}^a } / \prod_{j=0}^i ||\boldsymbol{W}_j||_2)}{\gamma}\right)^{-\frac{k}{2}} \\
&= \left( \frac{o(\tilde{s})}{o(\tilde{s}) - \sum_{i=0}^L (\sqrt{C_{i2}^a }  \prod_{j=i+1}^N||\boldsymbol{W}_j||_2)}\right)^{\frac{k}{2}} \\ 
\end{split}
\end{equation}

We can use the above to make predictions for the GE of the pruned network by noting that $\text{GE}(g_{a}) = \text{GE}(g_{b})\left( o(\tilde{s}) / (o(\tilde{s}) - \sum_{i=0}^L (\sqrt{C_{i2}^a }  \prod_{j=i+1}^N||\boldsymbol{W}_j||_2))\right)^{\frac{k}{2}}$ as we know that $\text{GE}(g_{b}) \approx 0.01$ for the unpruned network and we have managed to avoid the cumbersome $A$ parameter. Next we make the assumption that $k \approx 20$. Dimensionality values $20-40$ are common for the MNIST dataset and result from a simple dimensionality analysis using PCA. We also deviate slightly from our theory by using the minimum layerwise error $\min_i[\sqrt{C_{i2}^a }]$ for each sparsity level, as well as the average scores $\mathbb{E}_{s \sim S}[o(\boldsymbol{x},g(\boldsymbol{x}))]$. We plot the theoretical predictions for single layer pruning in Figure 4.a and the theoretical predictions for multilayer pruning in Figure 4.b. We see that, while loose, the theoretical predictions correctly capture qualitatively the behaviour of the GE. Specifically, layers, as predicted, are exponentially less robust with remaining layer depth. Also , as predicted, when pruning multiple layers the resulting GE is exponentially greater than the sum of the individual GEs.

\section{Conclusion}
In this paper we have presented an efficient pruning algorithm for fully connected layers of DNNs, based on difference of convex functions optimisation. Our algorithm is orders of magnitude faster than competing approaches while allowing for a controlled increase in the GE. We provided a theoretical analysis of the increase in GE resulting from bounded perturbations to the hidden layer weights, of which pruning is a special case. This analysis correctly predicts the previously observed phenomenon that network layers closer to the input are exponentially less robust to pruning compared to layers close to the output. Experiments on common feedforward architectures validated our results.

\bibliography{icml2018_submission}

\begin{thebibliography}{23}
\providecommand{\natexlab}[1]{#1}
\providecommand{\url}[1]{\texttt{#1}}
\expandafter\ifx\csname urlstyle\endcsname\relax
  \providecommand{\doi}[1]{doi: #1}\else
  \providecommand{\doi}{doi: \begingroup \urlstyle{rm}\Url}\fi

\bibitem[Aghasi et~al.(2016)Aghasi, Nguyen, and Romberg]{aghasi2016net}
Aghasi, Alireza, Nguyen, Nam, and Romberg, Justin.
\newblock Net-trim: A layer-wise convex pruning of deep neural networks.
\newblock \emph{arXiv preprint arXiv:1611.05162}, 2016.

\bibitem[Bartlett et~al.(2017)Bartlett, Foster, and
  Telgarsky]{bartlett2017spectrally}
Bartlett, Peter, Foster, Dylan~J, and Telgarsky, Matus.
\newblock Spectrally-normalized margin bounds for neural networks.
\newblock \emph{arXiv preprint arXiv:1706.08498}, 2017.

\bibitem[Blundell et~al.(2015)Blundell, Cornebise, Kavukcuoglu, and
  Wierstra]{blundell2015weight}
Blundell, Charles, Cornebise, Julien, Kavukcuoglu, Koray, and Wierstra, Daan.
\newblock Weight uncertainty in neural networks.
\newblock \emph{arXiv preprint arXiv:1505.05424}, 2015.

\bibitem[Chen et~al.(2015)Chen, Wilson, Tyree, Weinberger, and
  Chen]{chen2015compressing}
Chen, Wenlin, Wilson, James, Tyree, Stephen, Weinberger, Kilian, and Chen,
  Yixin.
\newblock Compressing neural networks with the hashing trick.
\newblock In \emph{International Conference on Machine Learning}, pp.\
  2285--2294, 2015.

\bibitem[Courbariaux et~al.(2016)Courbariaux, Hubara, Soudry, El-Yaniv, and
  Bengio]{courbariaux2016binarized}
Courbariaux, Matthieu, Hubara, Itay, Soudry, Daniel, El-Yaniv, Ran, and Bengio,
  Yoshua.
\newblock Binarized neural networks: Training deep neural networks with weights
  and activations constrained to+ 1 or-1.
\newblock \emph{arXiv preprint arXiv:1602.02830}, 2016.

\bibitem[Dong et~al.(2017)Dong, Chen, and Pan]{dong2017learning}
Dong, Xin, Chen, Shangyu, and Pan, Sinno~Jialin.
\newblock Learning to prune deep neural networks via layer-wise optimal brain
  surgeon.
\newblock \emph{arXiv preprint arXiv:1705.07565}, 2017.

\bibitem[Fawzi et~al.(2015)Fawzi, Davies, and Frossard]{fawzi2015dictionary}
Fawzi, Alhussein, Davies, Mike, and Frossard, Pascal.
\newblock Dictionary learning for fast classification based on
  soft-thresholding.
\newblock \emph{International Journal of Computer Vision}, 114\penalty0
  (2-3):\penalty0 306--321, 2015.

\bibitem[Han et~al.(2015{\natexlab{a}})Han, Mao, and Dally]{han2015deep}
Han, Song, Mao, Huizi, and Dally, William~J.
\newblock Deep compression: Compressing deep neural networks with pruning,
  trained quantization and huffman coding.
\newblock \emph{arXiv preprint arXiv:1510.00149}, 2015{\natexlab{a}}.

\bibitem[Han et~al.(2015{\natexlab{b}})Han, Pool, Tran, and
  Dally]{han2015learning}
Han, Song, Pool, Jeff, Tran, John, and Dally, William.
\newblock Learning both weights and connections for efficient neural network.
\newblock In \emph{Advances in Neural Information Processing Systems}, pp.\
  1135--1143, 2015{\natexlab{b}}.

\bibitem[Hou et~al.(2016)Hou, Yao, and Kwok]{hou2016loss}
Hou, Lu, Yao, Quanming, and Kwok, James~T.
\newblock Loss-aware binarization of deep networks.
\newblock \emph{arXiv preprint arXiv:1611.01600}, 2016.

\bibitem[Kim et~al.(2015)Kim, Park, Yoo, Choi, Yang, and
  Shin]{kim2015compression}
Kim, Yong-Deok, Park, Eunhyeok, Yoo, Sungjoo, Choi, Taelim, Yang, Lu, and Shin,
  Dongjun.
\newblock Compression of deep convolutional neural networks for fast and low
  power mobile applications.
\newblock \emph{arXiv preprint arXiv:1511.06530}, 2015.

\bibitem[LeCun et~al.(2015)LeCun, Bengio, and Hinton]{lecun2015deep}
LeCun, Yann, Bengio, Yoshua, and Hinton, Geoffrey.
\newblock Deep learning.
\newblock \emph{Nature}, 521\penalty0 (7553):\penalty0 436--444, 2015.

\bibitem[Louizos et~al.(2017)Louizos, Ullrich, and
  Welling]{louizos2017bayesian}
Louizos, Christos, Ullrich, Karen, and Welling, Max.
\newblock Bayesian compression for deep learning.
\newblock \emph{arXiv preprint arXiv:1705.08665}, 2017.

\bibitem[Molchanov et~al.(2017)Molchanov, Ashukha, and
  Vetrov]{molchanov2017variational}
Molchanov, Dmitry, Ashukha, Arsenii, and Vetrov, Dmitry.
\newblock Variational dropout sparsifies deep neural networks.
\newblock \emph{arXiv preprint arXiv:1701.05369}, 2017.

\bibitem[Neyshabur et~al.(2017)Neyshabur, Bhojanapalli, McAllester, and
  Srebro]{neyshabur2017pac}
Neyshabur, Behnam, Bhojanapalli, Srinadh, McAllester, David, and Srebro,
  Nathan.
\newblock A pac-bayesian approach to spectrally-normalized margin bounds for
  neural networks.
\newblock \emph{arXiv preprint arXiv:1707.09564}, 2017.

\bibitem[Nitanda(2014)]{nitanda2014stochastic}
Nitanda, Atsushi.
\newblock Stochastic proximal gradient descent with acceleration techniques.
\newblock In \emph{Advances in Neural Information Processing Systems}, pp.\
  1574--1582, 2014.

\bibitem[Novikov et~al.(2015)Novikov, Podoprikhin, Osokin, and
  Vetrov]{novikov2015tensorizing}
Novikov, Alexander, Podoprikhin, Dmitrii, Osokin, Anton, and Vetrov, Dmitry~P.
\newblock Tensorizing neural networks.
\newblock In \emph{Advances in Neural Information Processing Systems}, pp.\
  442--450, 2015.

\bibitem[Raghu et~al.(2016)Raghu, Poole, Kleinberg, Ganguli, and
  Sohl-Dickstein]{raghu2016expressive}
Raghu, Maithra, Poole, Ben, Kleinberg, Jon, Ganguli, Surya, and Sohl-Dickstein,
  Jascha.
\newblock On the expressive power of deep neural networks.
\newblock \emph{arXiv preprint arXiv:1606.05336}, 2016.

\bibitem[Sokolic et~al.(2017)Sokolic, Giryes, Sapiro, and
  Rodrigues]{sokolic2017robust}
Sokolic, Jure, Giryes, Raja, Sapiro, Guillermo, and Rodrigues, Miguel~RD.
\newblock Robust large margin deep neural networks.
\newblock \emph{IEEE Transactions on Signal Processing}, 2017.

\bibitem[Tao \& An(1997)Tao and An]{tao1997convex}
Tao, Pham~Dinh and An, Le Thi~Hoai.
\newblock Convex analysis approach to dc programming: Theory, algorithms and
  applications.
\newblock \emph{Acta Mathematica Vietnamica}, 22\penalty0 (1):\penalty0
  289--355, 1997.

\bibitem[Wolfe et~al.(2017)Wolfe, Sharma, Drude, and Raj]{wolfe2017incredible}
Wolfe, Nikolas, Sharma, Aditya, Drude, Lukas, and Raj, Bhiksha.
\newblock The incredible shrinking neural network: New perspectives on learning
  representations through the lens of pruning.
\newblock \emph{arXiv preprint arXiv:1701.04465}, 2017.

\bibitem[Xu \& Mannor(2012)Xu and Mannor]{xu2012robustness}
Xu, Huan and Mannor, Shie.
\newblock Robustness and generalization.
\newblock \emph{Machine learning}, 86\penalty0 (3):\penalty0 391--423, 2012.

\bibitem[Yang et~al.(2015)Yang, Moczulski, Denil, de~Freitas, Smola, Song, and
  Wang]{yang2015deep}
Yang, Zichao, Moczulski, Marcin, Denil, Misha, de~Freitas, Nando, Smola, Alex,
  Song, Le, and Wang, Ziyu.
\newblock Deep fried convnets.
\newblock In \emph{Proceedings of the IEEE International Conference on Computer
  Vision}, pp.\  1476--1483, 2015.

\end{thebibliography}
\bibliographystyle{icml2018}

\appendix

\onecolumn

\subsection*{A. Proof of theorem 3.1.}
We denote by $f^1(\cdot,\boldsymbol{W}^1)$ the original representation and by $f^2(\cdot,\boldsymbol{W}^2)$ the pruned representation.  We assume that, after training, $\forall s \in \mathcal{S}_m \: ||f^1(\boldsymbol{a},\boldsymbol{W}^1)-f^2(\boldsymbol{a},\boldsymbol{W}^2)||_2^2 \leq C_1$. Second, we assume that $\forall s \in \mathcal{S} \; \exists s_i \in \mathcal{S}_m \Rightarrow ||a-a_i||^2_2 \leq \epsilon $. Third the linear operators in $\boldsymbol{W}^1$ , $\boldsymbol{W}^2$ are frames with upper frame bounds $B_1$ , $B_2$ respectively.
The following two lemmas will be useful:

\begin{lemma}
The operator $f^1(\cdot,\boldsymbol{W}^1)$ is Lipschitz continuous with upper Lipschitz constant $B_1$.
\end{lemma}
\begin{proof}
See \citet{bruna2013signal} for details, the derivation is not entirely trivial due to the non-smoothness of the rectifier non-linearity.
\end{proof}

\begin{lemma}
The operator $f^2(\cdot,\boldsymbol{W}^2)$ is Lipschitz continuous with upper Lipschitz constant $B_2$.
\end{lemma}
\begin{proof}
We see that: $\frac{d}{dx} \rho(x) = \frac{d}{dx} \frac{1}{\beta}\text{log}(1+\text{exp}(\beta x)) = \frac{\text{exp} (\beta x)}{1+\text{exp}(\beta x)} \leq 1$. Therefore the smooth approximation to the rectifier non-linarity is Lipschitz smooth with Lipschitz constant $k=1$. Then $ || f^2(x,\boldsymbol{W}^2)-f^2(y,\boldsymbol{W}^2)||_2^2 \leq k||\boldsymbol{W}^2x-\boldsymbol{W}^2y||_2^2 \leq ||\boldsymbol{W}^2(x-y)||_2^2 \leq B_2 ||x-y||_2^2$.
\end{proof}

We drop the $\boldsymbol{W}^i$ from the layer notation for clarity. Using the triangle inequality
\begin{equation}
\begin{split}
||f^1(\boldsymbol{a})-f^2(\boldsymbol{a})||_2^2 & = ||f^1(\boldsymbol{a})+f^1(\boldsymbol{a}_i)-f^1(\boldsymbol{a}_i)-f^2(\boldsymbol{a})||_2^2 \\
& \leq ||f^1(\boldsymbol{a})-f^1(\boldsymbol{a}_i)||_2^2+||f^1(\boldsymbol{a}_i)-f^2(\boldsymbol{a})||_2^2 \\
& = ||f^1(\boldsymbol{a})-f^1(\boldsymbol{a}_i)||_2^2+||f^1(\boldsymbol{a}_i)+f^2(\boldsymbol{a}_i)-f^2(\boldsymbol{a}_i)-f^2(\boldsymbol{a})||_2^2 \\
& \leq ||f^1(\boldsymbol{a})-f^1(\boldsymbol{a}_i)||_2^2+||f^1(\boldsymbol{a}_i)-f^2(\boldsymbol{a}_i)||_2^2+||f^2(\boldsymbol{a}_i)-f^2(\boldsymbol{a})||_2^2 \\
& \leq B_1||\boldsymbol{a}_i-\boldsymbol{a}||^2_2+C + B_2||\boldsymbol{a}_i-\boldsymbol{a}||^2_2 \\ 
& = C_1 + (B_1+B_2)||\boldsymbol{a}_i-\boldsymbol{a}||^2_2 \\
& \leq C_1 + (B_1+B_2)\epsilon, \\
\end{split}
\end{equation}
where we used Lemma 6.1 and Lemma 6.2 in line 5.

\subsection*{B. Proof of theorem 3.2.}
We will proceed as follows. We first introduce some prior results which hold for the general class of robust classifiers. We will then give specific prior generalization error results for the case of classifiers operating on datapoints from $C_m$-regular manifolds. Afterwards we will provide prior results for the specific case of DNN clasifiers. Finally we will prove our novel generalization error bound and provide a link with prior bounds.  

\vspace{0.3 cm}

We first formalize robustness for generic classifiers $g(\boldsymbol{x})$. In the following we assume a loss function $l(g(\boldsymbol{x}),y)$ that is positive and bounded $0 \leq l(g(\boldsymbol{x}),y) \leq M$.
\begin{definition}
An algorithm $g(\boldsymbol{x})$ is $(K,\epsilon (\mathcal{S}_m))$ robust if $\mathcal{S}$ can be partitioned into K disjoint sets, denoted by $\{T_t\}_{t=1}^K$, such that $\forall s_i \in \mathcal{S}_m$, $\forall s \in \mathcal{S}$,

\begin{equation}
s_i,s \in T_t, \Rightarrow |l(g(\boldsymbol{x}_i),y_i)-l(g(\boldsymbol{x}),y)| \leq \epsilon(\mathcal{S}_m).
\end{equation}
\end{definition}

Now let $\hat{l}(\cdot)$ and $l_{\text{emp}}(\cdot)$ denote the expected error and the training error, i.e,

\begin{equation}
\hat{l}(g) \triangleq \mathbb{E}_{s \sim S}l(g(\boldsymbol{x}),y); \; \; \; l_{\text{emp}}(g) \triangleq \frac{1}{m} \sum_{s_i \in \mathcal{S}_m}l(q(\boldsymbol{x}_i),y_i)
\end{equation}

we can then state the following theorem from \citet{xu2012robustness}:

\begin{theorem}
If $\mathcal{S}_m$ consists of $m$ i.i.d. samples, and $g(\boldsymbol{x})$ is $(K,\epsilon (\mathcal{S}_m))$-robust, then for any $\delta > 0$, with probability at least $1-\delta$,

\begin{equation}
GE(g)=|\hat{l}(g) - l_{\text{emp}}(g)| \leq \epsilon (\mathcal{S}_m) + M \sqrt{\frac{2K\text{ln}2+2\text{ln}(1/ \delta )}{m} }.
\end{equation}

\end{theorem}

The above generic bound can be specified for the case of $C_m$-regular manifolds as in \citet{sokolic2017robust}. We recall the definition of the sample margin $\gamma(s_i)$ as well as the following theorem:

\begin{theorem}
If there exists $\gamma$ such that

\begin{equation}
\gamma(s_i) > \gamma > 0 \; \forall s_i \in S_m,
\end{equation}
then the classifier $g(\boldsymbol{x})$ is $(N_{\mathcal{Y}} \cdot \mathcal{N}(\mathcal{X};d,\gamma / 2 ),0 )$-robust.
\end{theorem}

By direct substitution of the above result and the definiton of a $C_m$-regular manifold into Theorem 6.3 we get: 

\begin{corollary}
Assume that $\mathcal{X}$ is a (subset of) $C_M$ regular $k-$dimensional manifold, where $\mathcal{N}(\mathcal{X};d,\rho)\leq(\frac{C_M}{\rho})^k$. Assume also that classifier $g(\boldsymbol{x})$ achieves a classification margin $\gamma$ and take $l(g(\boldsymbol{x}),y)$ to be the $0-1$ loss. Then for any $\delta > 0$, with probability at least $1-\delta$,
\begin{equation}
GE(g) \leq \sqrt{\frac{\text{log}(2)\cdot N_{\mathcal{Y}} \cdot 2^{k+1} \cdot (C_M)^k }{\gamma^k m} }+\sqrt{\frac{2 \text{log}(1/\delta) }{m} }.
\end{equation}

\end{corollary}

Note that in the above we have used the fact that $l(g(\boldsymbol{x}),y) \leq 1$ and therefore $M=1$. The above holds for a wide range of algorithms that includes as an example SVMs. We are now ready to specify the above bound for the case of DNNs, adapted from \citet{sokolic2017robust},

\begin{theorem}
Assume that a DNN classifier $g(\boldsymbol{x})$, as defined in equation 8, and let $\tilde{\boldsymbol{x}}$ be the training sample with the smallest score $o(\tilde{s})>0$. Then the classification margin is bounded as
\begin{equation}
\gamma(s_i)\geq \frac{o(\tilde{s})}{\prod_i ||\boldsymbol{W}_i||_2 } = \gamma.
\end{equation}
\end{theorem}

We now prove our main result. We will denote by $ \tilde{\boldsymbol{x}} = \text{arg min}_{s_i \in S_m}\text{min}_{j \neq g(\boldsymbol{x}_i)} \boldsymbol{v}^{T}_{g(\boldsymbol{x}_i) j} f(\boldsymbol{x}_i) $ the training sample with the smallest score. For this training sample we will denote $j^{\star} = \text{arg min}_{j \neq g(\tilde{\boldsymbol{x}})} \boldsymbol{v}^{T}_{g(\tilde{\boldsymbol{x}}) j} f(\tilde{\boldsymbol{x}})$ the second best guess of the classifier $g(\cdot)$. Throughout the proof, we will use the notation $\boldsymbol{v}_{ij}=\sqrt{2}(\boldsymbol{\delta}_i-\boldsymbol{\delta}_j)$. 

\vspace{0.3 cm}

First we assume the score $o_1(\tilde{\boldsymbol{x}},g_1(\tilde{\boldsymbol{x}}))$ of the point $\tilde{\boldsymbol{x}}$ for the original classifier $g_1(\boldsymbol{x})$. Then, for the second classifier $g_2(\boldsymbol{x})$, we take a point $\boldsymbol{x}^{\star}$ that lies on the decision boundary between $g_2(\tilde{\boldsymbol{x}})$ and $j^{\star}$ such that $o_2(\boldsymbol{x}^{\star},g_2(\tilde{\boldsymbol{x}}))=0$. We assume for simplicity that, after pruning, the classification decisions do not change such that $g_1(\tilde{\boldsymbol{x}}) = g_2(\tilde{\boldsymbol{x}})$. We then make the following calculations

\begin{equation}
\begin{split}
o_1(\tilde{\boldsymbol{x}},g_1(\tilde{\boldsymbol{x}})) & = o_1(\tilde{\boldsymbol{x}},g_1(\tilde{\boldsymbol{x}})) - o_2(\boldsymbol{x}^{\star},g_2(\tilde{\boldsymbol{x}})) = \boldsymbol{v}^{T}_{g_1(\tilde{\boldsymbol{x}}) j^{\star}}f^1(\tilde{\boldsymbol{x}})-\boldsymbol{v}^{T}_{g_2(\tilde{\boldsymbol{x}}) j^{\star}}f^2(\boldsymbol{x}^{\star}) \\
& = \boldsymbol{v}^{T}_{g_2(\tilde{\boldsymbol{x}}) j^{\star}}(f^1(\tilde{\boldsymbol{x}})-f^2(\boldsymbol{x}^{\star})) \\
& \leq ||\boldsymbol{v}^{T}_{g_2(\tilde{\boldsymbol{x}}) j^{\star}}||_2||f^1(\tilde{\boldsymbol{x}})-f^2(\boldsymbol{x}^{\star})||_2 = ||f^1_L(\tilde{\boldsymbol{x}})-f^2_L(\boldsymbol{x}^{\star})||_2 \\
& \leq \prod_{i>i^{\star}} ||\boldsymbol{W}_i||_2||f^1_{i^{\star}}(\tilde{\boldsymbol{x}})-f^2_{i^{\star}}(\boldsymbol{x}^{\star})||_2 \\
& \leq \prod_{i>i^{\star}} ||\boldsymbol{W}_i||_2 \{ ||f^1_{i^{\star}}(\tilde{\boldsymbol{x}})-f^1_{i^{\star}}(\boldsymbol{x}^{\star})||_2 + ||f^1_{i^{\star}}(\boldsymbol{x}^{\star})-f^2_{i^{\star}}(\boldsymbol{x}^{\star})||_2 \} \\
& \leq \prod_{i>i^{\star}} ||\boldsymbol{W}_i||_2 \{ ||f^1_{i^{\star}}(\tilde{\boldsymbol{x}})-f^1_{i^{\star}}(\boldsymbol{x}^{\star})||_2 + \sqrt{C_2} \} \\
& \leq \prod_{i} ||\boldsymbol{W}_i||_2  ||\tilde{\boldsymbol{x}}-\boldsymbol{x}^{\star}||_2 + \prod_{i>i^{\star}} ||\boldsymbol{W}_i||_2 \sqrt{C_2} \\
& \leq \prod_{i} ||\boldsymbol{W}_i||_2  \gamma_2(s_i) + \prod_{i>i^{\star}} ||\boldsymbol{W}_i||_2 \sqrt{C_2}, \\
\end{split}
\end{equation}

where we used Theorem 3.1 in line 5, since $\boldsymbol{x}^{\star}$ is not a training sample. From the above we can therefore write

\begin{equation}
\frac{o_1(\tilde{\boldsymbol{x}},g_1(\tilde{\boldsymbol{x}}))-\sqrt{C_2}\prod_{i>i^{\star}} ||\boldsymbol{W}_i||_2}{\prod_{i} ||\boldsymbol{W}_i||_2} \leq  \gamma_2(\tilde{\boldsymbol{x}}).
\end{equation}

By following the derivation of the margin from the original paper \citet{sokolic2017robust} and taking into account the definition of the margin we know that

\begin{equation}
\gamma = \frac{o_1(\tilde{\boldsymbol{x}},g_1(\tilde{\boldsymbol{x}}))}{ \prod_{i} ||\boldsymbol{W}_i||_2} \leq \gamma_1(\tilde{\boldsymbol{x}}).
\end{equation}

Therefore we can finally write

\begin{equation}
\gamma - \frac{\sqrt{C_2}\prod_{i>i^{\star}} ||\boldsymbol{W}_i||_2}{\prod_{i} ||\boldsymbol{W}_i||_2} \leq  \gamma_2(\tilde{\boldsymbol{x}}).
\end{equation}

The theorem follows from direct application of Corollary 3.1.1. 
Note that if $\gamma - \frac{\sqrt{C_2}\prod_{i>i^{\star}} ||\boldsymbol{W}_i||_2}{\prod_{i} ||\boldsymbol{W}_i||_2} < 0$ the derived bound becomes vacuous, as by definition $0 \leq \gamma_2(\tilde{\boldsymbol{x}})$.

\subsection*{C. Proof of theorem 3.3.}
We start as in theorem 3.2 by assuming the score $o_1(\tilde{\boldsymbol{x}},g_1(\tilde{\boldsymbol{x}}))$ of the point $\tilde{\boldsymbol{x}}$ for the original classifier $g_1(\boldsymbol{x})$. Then, for the second classifier $g_2(\boldsymbol{x})$, we take a point $\boldsymbol{x}^{\star}$ that lies on the decision boundary between $g_2(\tilde{\boldsymbol{x}})$ and $j^{\star}$ such that $o_2(\boldsymbol{x}^{\star},g_2(\tilde{\boldsymbol{x}}))=0$. We assume as before that the classification decisions do not change such that $g_1(\tilde{\boldsymbol{x}}) = g_2(\tilde{\boldsymbol{x}})$. We write

\begin{equation}
\begin{split}
o_1(\tilde{\boldsymbol{x}},g_1(\tilde{\boldsymbol{x}})) & = o_1(\tilde{\boldsymbol{x}},g_1(\tilde{\boldsymbol{x}})) - o_2(\boldsymbol{x}^{\star},g_2(\tilde{\boldsymbol{x}})) = \boldsymbol{v}^{T}_{g_1(\tilde{\boldsymbol{x}}) j^{\star}}f^1(\tilde{\boldsymbol{x}})-\boldsymbol{v}^{T}_{g_2(\tilde{\boldsymbol{x}}) j^{\star}}f^2(\boldsymbol{x}^{\star}) \\
& = \boldsymbol{v}^{T}_{g_2(\tilde{\boldsymbol{x}}) j^{\star}}(f^1(\tilde{\boldsymbol{x}})-f^2(\boldsymbol{x}^{\star})) \\
& \leq ||\boldsymbol{v}^{T}_{g_2(\tilde{\boldsymbol{x}}) j^{\star}}||_2||f^1(\tilde{\boldsymbol{x}})-f^2(\boldsymbol{x}^{\star})||_2 = ||f^1_L(\tilde{\boldsymbol{x}})-f^2_L(\boldsymbol{x}^{\star})||_2 \\
& \leq ||f^1_{L}(\tilde{\boldsymbol{x}})-f^1_{L}(\boldsymbol{x}^{\star})||_2 + ||f^1_{L}(\boldsymbol{x}^{\star})-f^2_{L}(\boldsymbol{x}^{\star})||_2  \\
& \leq ||f^1_{L}(\tilde{\boldsymbol{x}})-f^1_{L}(\boldsymbol{x}^{\star})||_2 + \sqrt{C_{L2}}  \\
& \leq ||\boldsymbol{W}_L ||_2 ||f^1_{L-1}(\tilde{\boldsymbol{x}})-f^2_{L-1}(\boldsymbol{x}^{\star})||_2 + \sqrt{C_{L2}}  \\
& \leq ||\boldsymbol{W}_L ||_2  \{ ||f^1_{L-1}(\tilde{\boldsymbol{x}})-f^1_{L-1}(\boldsymbol{x}^{\star})||_2 + ||f^1_{L-1}(\boldsymbol{x}^{\star})-f^2_{L-1}(\boldsymbol{x}^{\star})||_2 \} + \sqrt{C_{L2}}  \\
& \leq ||\boldsymbol{W}_L ||_2  \{ ||f^1_{L-1}(\tilde{\boldsymbol{x}})-f^1_{L-1}(\boldsymbol{x}^{\star})||_2 + \sqrt{C_{L-1,2}} \} + \sqrt{C_{L2}}  \\
& \leq ||\boldsymbol{W}_L ||_2   ||f^1_{L-1}(\tilde{\boldsymbol{x}})-f^1_{L-1}(\boldsymbol{x}^{\star})||_2 + ||\boldsymbol{W}_L ||_2 \sqrt{C_{L-1,2}}  + \sqrt{C_{L2}}  \\
& \leq ... \\
& \leq \prod_{i} ||\boldsymbol{W}_i||_2  ||\tilde{\boldsymbol{x}}-\boldsymbol{x}^{\star}||_2 + \sum_{i=0}^L \sqrt{C_{i2}} \prod_{j=i+1}^{L} ||\boldsymbol{W}_j||_2 \\
& \leq \prod_{i} ||\boldsymbol{W}_i||_2  \gamma_2(s_i) + \sum_{i=0}^L \sqrt{C_{i2}} \prod_{j=i+1}^{L} ||\boldsymbol{W}_j||_2. \\
\end{split}
\end{equation}

We can then write

\begin{equation}
\frac{o_1(\tilde{\boldsymbol{x}},g_1(\tilde{\boldsymbol{x}}))- \sum_{i=0}^L \sqrt{C_{i2} } \prod_{j=i+1}^L||\boldsymbol{W}_j||_2}{ \prod_i||\boldsymbol{W}_i||_2} \leq  \gamma_2(\tilde{\boldsymbol{x}}).
\end{equation}

Then as before

\begin{equation}
\gamma - \frac{\sum_{i=0}^L \sqrt{C_{i2} } \prod_{j=i+1}^L||\boldsymbol{W}_j||_2}{ \prod_i||\boldsymbol{W}_i||_2} \leq  \gamma_2(\tilde{\boldsymbol{x}}).
\end{equation}

The theorem follows from direct application of Corollary 3.1.1.

\end{document}